\documentclass[runningheads,a4paper]{llncs}
\usepackage{amsmath} \usepackage{amstext} \usepackage{amsbsy}
\usepackage{prettyref} \usepackage{latexsym} \usepackage{amssymb}
\usepackage{graphicx} \usepackage{hyperref} \usepackage{braket}
\usepackage{aas_macros}

\usepackage{multicol}

\let\llncssubparagraph\subparagraph
\let\subparagraph\paragraph
\usepackage[compact]{titlesec} \titlespacing{\section}{0pt}{*3}{*2}
\titlespacing{\subsection}{0pt}{*1}{*0}
\titlespacing{\subsubsection}{0pt}{*0}{*0}
\let\subparagraph\llncssubparagraph

\DeclareMathOperator*{\argmin}{arg\,min}

\newrefformat{alg}{Algorithm~\ref{#1}}
\newrefformat{eq}{Equation~\ref{#1}}
\newrefformat{lem}{Lemma~\ref{#1}}
\newrefformat{def}{Definition~\ref{#1}}
\newrefformat{cor}{Corollary~\ref{#1}}
\newrefformat{chp}{Chapter~\ref{#1}}
\newrefformat{sec}{Section~\ref{#1}}
\newrefformat{apx}{Appendix~\ref{#1}}
\newrefformat{tab}{Table~\ref{#1}} \newrefformat{fig}{Fig.~\ref{#1}}

\newcommand{\field}[1]{\mathbb{#1}} 
\newcommand{\Z}{\field{Z}}

\newcommand{\CJS}{\mathrm{CJS}} \newcommand{\CJV}{\mathrm{CJV}}
\newcommand{\CJE}{\mathrm{CJE}} \newcommand{\CJTE}{\mathrm{CJTE}}

\titlerunning{Physical Complexity and Limits} \title{Ultimate
  Intelligence Part II: Physical Complexity and Limits of Inductive
  Inference Systems}

\author{Eray \"Ozkural}

\date{\today}

\begin{document}

\mainmatter 


\institute{G\"{o}k Us Sibernetik Ar\&Ge Ltd. \c{S}ti.}

\maketitle

\begin{abstract}
  We continue our analysis of volume and energy measures that are
  appropriate for quantifying inductive inference systems.  We extend
  logical depth and conceptual jump size measures in AIT to stochastic
  problems, and physical measures that involve volume and energy.  We
  introduce a graphical model of computational complexity that we
  believe to be appropriate for intelligent machines.  We show several
  asymptotic relations between energy, logical depth and volume of
  computation for inductive inference. In particular, we arrive at a
  ``black-hole equation'' of inductive inference, which relates
  energy, volume, space, and algorithmic information for an optimal
  inductive inference solution.  We introduce energy-bounded
  algorithmic entropy.  We briefly apply our ideas to the physical
  limits of intelligent computation in our universe.
\end{abstract}

``Everything must be made as simple as possible. But not
  simpler.''
\begin{flushright}--- Albert Einstein \end{flushright}

\section{Introduction}

We initiated the ultimate intelligence research program in 2014
inspired by Seth Lloyd's similarly titled article on the ultimate
physical limits to computation \cite{Lloyd:Ultimate}, intended as a
book-length treatment of the theory of general-purpose AI. In similar
spirit to Lloyd's research, we investigate the ultimate physical
limits and conditions of intelligence. A main motivation is to extend
the theory of intelligence using physical units, emphasizing the
physicalism inherent in computer science.  This is the second
installation of the paper series, the first part \cite{ozkural-agi15} proposed that
universal induction theory is physically complete arguing that the
algorithmic entropy of a physical stochastic source is always finite,
and argued that if we choose the laws of physics as the reference
machine, the loophole in algorithmic information theory (AIT) of
choosing a reference machine is closed.  We also introduced several
new physically meaningful complexity measures adequate for reasoning
about intelligent machinery using the concepts of minimum volume,
energy and action, which are applicable to both classical and quantum
computers. Probably the most important of the new measures was the
minimum energy required to physically transmit a message. The minimum
energy complexity also naturally leads to an energy prior,
complementing the speed prior \cite{Schmidhuber2002} which inspired
our work on incorporating physical resource limits to inductive
inference theory.

In this part, we generalize logical depth and conceptual jump size to
stochastic sources and consider the influence of volume, space and
energy.  We consider the energy efficiency of computing as an
important parameter for an intelligent system, forgoing other details
of a universal induction approximation. We thus relate the ultimate
limits of intelligence to physical limits of computation.

\section{Notation and Background}

Let us recall Solomonoff's universal distribution \cite{alp1}.  Let $U$ be a
universal computer which runs programs with a prefix-free encoding
like LISP; $y=U(x)$ denotes that the output of program $x$ on $U$ is
$y$ where $x$ and $y$ are bit strings. \footnote{A prefix-free code is
  a set of codes in which no code is a prefix of another. A computer
  file uses a prefix-free code, ending with an EOF symbol, thus, most
  reasonable programming languages are prefix-free. 
}  
Any unspecified variable or function is assumed to be represented as a
bit string. $|x|$ denotes the length of a bit-string $x$.
$f(\cdot)$ refers to function $f$ rather than its application.

The algorithmic probability that a bit string $x \in \{0,1\}^+$ is
generated by a random program $\pi \in \{0,1\}^+$ of $U$ is:
\begin{equation}
  \label{eq:alp}
  P_U(x) = \sum_{U(\pi) \in x(0+1)^* \wedge \pi \in  \{0,1\}^+} 2^{-|\pi|}
\end{equation}
which conforms to Kolmogorov's axioms \cite{levin-thesis}.  $P_U(x)$
considers any continuation of $x$, taking into account non-terminating
programs.\footnote{We used the regular expression notation in language
  theory.}  
$P_U$ is also called the universal prior for it
may be used as the prior in Bayesian inference, for any data can be
encoded as a bit string.

We also give the basic definition of Algorithmic Information Theory
(AIT), where the algorithmic entropy, or complexity of a bit string $x
\in \{0,1\}^+$ is
\begin{equation}
  \label{eq:algo-entropy}
  H_U(x) = \min( \{ |\pi| \ | \  U(\pi)=x \} )
\end{equation}

We shall now briefly recall the well-known Solomonoff induction method
\cite{alp1,alp2}.  Universal sequence induction method of Solomonoff
works on bit strings $x$ drawn from a stochastic source $\mu$.
\prettyref{eq:alp} is a semi-measure, but that is easily overcome as
we can normalize it.  We merely normalize sequence probabilities
\begin{alignat}{3}
  \label{eq:normalization}
  P'_U(x0)=\frac{P_U(x0).P'_U(x)}{P_U(x0)+P_U(x1)} && \quad
  P'_U(x1)=\frac{P_U(x1).P_U'(x)}{P_U(x0)+P_U(x1)}
\end{alignat}
eliminating irrelevant programs and ensuring that the probabilities
sum to $1$, from which point on $P'_U(x0|x) = P'_U(x0)/P'_U(x)$ yields
an accurate prediction. The error bound for this method is the best
known for any such induction method.  The total expected squared error
between $P'_U(x)$ and $\mu$ is
\begin{equation}
  \label{eq:convergence}
  E_P \left[ \sum_{m=1}^n{(P'_U({a_{m+1}=1}|a_1a_2...a_m) -
      \mu({a_{m+1}=1}|a_1a_2...a_m))^2}  \right] 
  \leq - \frac{1}{2} \ln{P_U(\mu)} 
\end{equation}
which is less than $-1/2\ln{P'_U}(\mu)$ according to the convergence
theorem proven in \cite{solcomplexity}, and it is roughly
$H_U(\mu)\ln2$ \cite{solomonoff-threekinds}.  Naturally, this method
can only work if the algorithmic complexity of the stochastic source
$H_U(\mu)$ is finite, i.e., the source has a computable probability
distribution. The convergence theorem is quite significant, because it
shows that Solomonoff induction has the best generalization
performance among all prediction methods. In particular, the total
error is expected to be a constant independent of the input, and the
error rate will thus rapidly decrease with increasing input size.

Operator induction is a general form of supervised machine learning
where we learn a stochastic map from question and answer pairs $q_i,
a_i$ sampled from a (computable) stochastic source $\mu$.  Operator
induction can be solved by finding in available time a set of operator
models $O^j(\cdot|\cdot)$ such that the following goodness of fit is
maximized
\begin{equation}
  \label{eq:opind-gof}
  \Psi = \sum_j{\psi^j_n}
\end{equation}
for a stochastic source $\mu$ where each term in the summation is 
\begin{equation}
  \label{eq:opind-gof-term}
  \psi^j_n= 2^{-|O^j(\cdot|\cdot)|}\prod_{i=1}^n{O^j(a_i|q_i)}.
\end{equation}
$q_i$ and $a_i$ are question/answer pairs in the input dataset, and
$O^j$ is a computable conditional pdf (cpdf) in
\prettyref{eq:opind-gof-term}.  We can use the found operators to
predict unseen data \cite{solomonoff-threekinds}
\begin{equation}
  \label{eq:opind-pred}
  P_U(a_{n+1}|q_{n+1}) = \sum_{j=1}^n\psi^j_nO^j(a_{n+1}|q_{n+1})
\end{equation}
The goodness of fit in this case strikes a balance between high a
priori probability and reproduction of data like in minimum message
length (MML) method, yet uses a universal mixture like in sequence
induction.  The convergence theorem for operator induction was proven
in \cite{solomonoff-progress} using Hutter's extension to arbitrary
alphabet.

Operator induction infers a generalized conditional probability
density function (cpdf), and Solomonoff argues that it can be used to
teach a computer anything.  For instance, we can train the
question/answer system with physics questions and answers, and the
system would then be able to answer a new physics question, dependent
upon how much has been taught in the examples; a future user could ask
the system to describe a physics theory that unifies quantum mechanics
and general relativity, given the solutions of every mathematics and
physics problem ever solved in literature.  Solomonoff's original
training sequence plan proposed to instruct the system first with an
English subset and basic algebra, and then venture into more complex
subjects.  The generality of operator induction is partly due to the
fact that it can be used to learn any kind of association, i.e., it
models an ideal content-addressable memory, but it also generalizes
any kind of law therein implicitly, that is why it can learn an
implicit principle (such as of syntax) from linguistic input, enabling
the system to acquire language; it can also model complex translation
problems, and all manners of problems that require additional
reasoning (computation). In other words, it is a universal problem
solver model.  It is also the most general of the three kinds of
induction, which are sequence, set, and operator induction, and the
closest to machine learning literature. The popular applications of
speech and image recognition are covered by operator induction model,
as is the wealth of pattern recognition applications, such as
describing a scene in English. We think that, therefore, operator
induction is an AI-complete problem -- as hard as solving the
human-level AI problem in general. It is with this in mind that we
analyze the asymptotic behavior of an optimal solution to operator
induction problem.
 
\section{Physical Limits to Universal Induction}

In this section, we elucidate the physical resource limits in the
context of a hypothetical optimal solution to operator induction. We
first extend Bennett's logical depth and conceptual jump size to the
case of operator induction, and show a new relation between expected
simulation time of the universal mixture and conceptual jump size.  We
then introduce a new graphical model of computational complexity which
we use to derive the relations among physical resource bounds.  We
introduce a new definition of physical computation which we call
self-contained computation, which is a physical counterpart to
self-delimiting program.  The discovery of these basic bounds, and
relations, exact, and asymptotic, give meaning to the complexity
definitions of Part I.

Please note that Schmidhuber disagrees with the model of the
stochastic source as a computable pdf \cite{Schmidhuber2002}, 
but Part I contained a strong argument that this was indeed the case.
 A stochastic source cannot have a pdf that is computable only in the limit, if that were the
case, it could have a random pdf, which would have infinite
algorithmic information content, and that is clearly contradicted by
the main conclusion of Part I. A stochastic source cannot be
semi-computable, because it would eventually run out of energy and
hence the ability to generate further quantum entropy, especially the
self-contained computations of this section. That is the reason we had
introduced self-contained computation notion at any rate.  Note also
that Schmidhuber agrees that quantum entropy does not
accumulate to make the world incompressible in general, therefore we
consider his proposal that we should view a cpdf as computable in the
limit as too weak an assumption. As with Part I, the analysis of this
section is extensible to quantum computers, which is beyond the scope
of the present article.

\subsection{Logical depth and conceptual jump size}

Conceptual Jump Size (CJS) is the time required by an incremental
inductive inference system to learn a new concept, and it increases
exponentially in proportion to the algorithmic information content of the
concept to be learned relative to the concepts already known 
\cite{solomonoff-incremental}. 
The physical limits to OOPS based on Conceptual Jump Size were
examined in \cite{oops}. Here, we give a more detailed treatment. Let
$\pi^*$ be the computable cpdf that exactly simulates $\mu$ with
respect to $U$, for operator induction.
\begin{equation}
  \label{eq:minimal}
  \pi^* = \argmin_{\pi_j}(\{ |\pi_j| \ | \ \forall x,y \in \{0,1\}^*:  U(\pi_j,x,y)=\mu(x | y) \})
\end{equation}
The conceptual jump size of inductive inference ($\CJS$) can be
defined with respect to the optimal solution program using Levin
search \cite{sol-perfect}:
\begin{equation}
  \label{eq:cjs}
  \CJS(\mu) = \frac{t(\pi^*)} { P(\pi^*)} \leq  2.\CJS(\mu) 
\end{equation}
where $t(\cdot)$ is the running time of a program on $U$.
\begin{align}
  \label{eq:time}
  H_U(\pi^*) &= -\log_2{P_U(\pi^*)} = -\log_2{P_U(\mu)}\\
  \label{eq:time2}
  t(\mu) &\leq t(\pi^*) 2^{H_U(\mu)+1}
\end{align}
where $t(\mu)$ is the time for solving an induction problem from
source $\mu$ with sufficient input complexity ($>> H_U(\mu)$), we
observe that the asymptotic complexity is
\begin{align}
  \label{eq:time3}
  t(\mu) = O(2^{H_U(\mu)})
\end{align}
for fixed $t(\pi^*)$.  Note that $t(\pi^*)$ corresponds to the
\emph{stochastic} extension of Bennett's logical depth
\cite{bennett88logical}, which was defined as:
``the running time of the minimal program that computes $x$''.
Let us recall that the minimal program is essentially unique, a
polytope in program space \cite{chaitin-ait}.
\begin{definition}
  Stochastic logical depth is the running time of the minimal program
  that accurately simulates a stochastic source $\mu$.
  \begin{equation}
    \label{eq:logicaldepth}
    L_U(\mu) =  t(\pi^*)
  \end{equation}
\end{definition}
which, with \prettyref{eq:time2}, entails our first bound.
\begin{lemma}
\begin{align}
  \label{eq:time3}
  t(\mu) &\leq L_U(\mu). 2^{H_U(\mu)+1}
\end{align}
\end{lemma}

\begin{lemma}
  $\CJS$ is related to the \emph{expectation} of the simulation time
  of the universal mixture.
  \begin{equation}
    \label{eq:logicaldepth}
    CJS(\mu) \leq  \sum_{U(\pi) \in x(0+1)^*} t(\pi).2^{-|\pi|}
    = E_{P_U}[\{ t(\pi)\ | \ U(\pi) \in x(0+1)^* \}]
  \end{equation}
  where $x$ is the input data to sequence induction, without loss of
  generality.
\end{lemma}
\begin{proof}
  Rewrite as 
$  t(\pi^*) 2^{|-\pi^*|} \leq \sum_{U(\pi) \in  x(0+1)^*} t(\pi).2^{-|\pi|} $.
  Observe that left-hand side of the inequality is merely a term in
  the summation in the right.
\end{proof}

\subsection{A Graphical Analysis of Intelligent Computation}

Let us introduce a graphical model of computational complexity that
will help us visualize physical complexity relations that will be
investigated. We do not model the computation itself, we just
enumerate the physical resources required. Present treatment is merely
classical computation over sequential circuits.
\begin{definition}
  \label{def:lattice}
  Let the computation be represented by a directed bi-partite graph
  $G=(V,E)$ where vertices are partitioned into $V_O$ and $V_M$ which
  correspond to primitive operations and memory cells respectively,
  $V=V_O \cup V_M, V_O \cap V_M = \emptyset$.  Function $ t: V \cup E
  \rightarrow \Z$ assigns time to vertices and edges.
  \footnote{Time as discrete timestamps, as opposed to duration.}
  Edges
  correspond to causal dependencies. $I \subset V$ and $O \subset V$
  correspond to input and output vertices interacting with the rest of
  the world. We denote acccess to vertex subsets with functions over
  $G$, e.g., $I(G)$.
\end{definition}
\prettyref{def:lattice} is a low-level computational complexity model
where the physical resources consumed by any operation, memory cell,
and edge are the same for the sake of simplicity. Let $v_u$ be the
unit space-time volume, $e_u$ be the unit energy, and $s_u$ be the
unit space.

\begin{definition}
  \label{def:volume}
  Let the volume of computation be defined as $V_U(\pi)$ which
  measures the space-time volume of computation of $\pi$ on $U$ in
  physical units, i.e., $m^3.sec$.
\end{definition}
For \prettyref{def:lattice}, it is $(|V(G)|+|E(G)|).v_u$.  Volume of
computation measures the extent of the space-time region occupied by
the dynamical evolution of the computation of $\pi$ on $U$. We do not
consider the theory of relativity.
For instance, the space of a Turing Machine is the Instantaneous
Description (ID) of it, and its time corresponds to $Z^+$. A Turing
Machine derivation that has an ID of length $i$ at time $i$ and takes
$t$ steps to complete would have a volume of
$t.(t+1)/2$.\footnote{If the derivation is $A \rightarrow AA
  \rightarrow AAA$, it has $1+2+3 = 6$ volume.}
\begin{definition}
  \label{def:energy}
  Let the energy of computation be defined as $E_U(\pi)$ which
  measures the total energy required by computation of $\pi$ on $U$ in
  physical units, e.g, $J$.
\end{definition}
For \prettyref{def:lattice}, it is $E_U(\pi) = (|V(G)|+|E(G)|).e_u$.
\begin{definition}
  \label{eq:space}
  Let the space of computation be defined as $S_U(\pi)$ which measures
  the maximum volume of a synchronous slice of the space-time of
  computation $\pi$ on $U$ in physical units, e.g., $m^3$.
\end{definition}
For \prettyref{def:lattice}, it is
\begin{equation}
  \max_{i \in \Z}\{| \{x \in \{V(G) \cup E(G)\}  | \ t(x)=i \} |\}.s_u
\end{equation}
\begin{definition}
  \label{eq:self-contained}
  In a self-contained physical computation all the physical resources
  required by computation should be contained within the volume of
  computation.
\end{definition}
Therefore, we do not allow a self-contained physical computation to
send queries over the internet, or use a power cord, for instance.

Using these new more general concepts, we measure the conceptual jump
size in space-time volume rather than time (space-time extent might be
a more accurate term). Algorithmic complexity remains the same, as the
length of a program readily generalizes to space-time volume of
program at the input boundary of computation, which would be $V_0(G) \triangleq |I(G)
\cap V_M(G)|.v_u$ for \prettyref{def:lattice}.
If $y=U(x)$, bitstring $x$ and y correspond to $I(G)$, and $O(G)$
respectively. A program $\pi$ corresponds to a vertex set $V_\pi
\subseteq I(G)$ usually, and its size is denoted as $V_0(\pi)$. We use
bitstrings for data and programs below, but measure their sizes in
physical units using this notation. It is possible to eliminate
bit strings altogether using a volume prior, we mix notations only for ease
of understanding.

Let us generalize logical depth to the logical volume of a bit string
$x$:
\begin{equation}
  \label{eq:logicalvol1}
  L^V_U(x) \triangleq  V_U( \argmin_{\pi} \{ V_0(\pi) \ | \ U(\pi) \in x(0+1)^* \}    )
\end{equation}

Let us also generalize stochastic logical depth to stochastic logical
volume:
\begin{equation}
  \label{eq:logicalvol2}
  L^V_U(\mu) \triangleq  V_U(\pi^*)
\end{equation} 
which entails that Conceptual Jump Volume (CJV), and logical volume
$V_U$ of a stochastic source may be defined analogously to CJS
\begin{equation}
  \label{eq:time3}
  \CJV(\mu) \triangleq L^V_U(\mu). 2^{H_U(\mu)} \leq  V_U(\mu) \leq 2.\CJV(\mu)
\end{equation}
where left-hand side corresponds to space-time extent variant of
$\CJS$.  
Likewise, we define logical energy for a bit string, and 
stochastic logical energy:
\begin{align}
  \label{eq:logicaleng1}
  L^E_U(x) &\triangleq  E_U( \argmin_{\pi} \{ V_0(\pi) \ | \ U(\pi) \in x(0+1)^* \}  )
& L^E_U(\mu) &\triangleq  E_U(\pi^*)
\end{align} 
Which brings us to an energy based statement of conceptual
jump size, that we term conceptual jump energy, or conceptual gap energy:
\begin{lemma}
  $\CJE(\mu) \triangleq E_U(\pi^*).2^{H_U(\mu)} \leq E_U(\mu) \leq 2.CJE(\mu)$.
\end{lemma}
The inequality holds since we can use $E_U(\cdot)$ bounds in universal search
instead of time.
We now show an interesting relation which is the case for
self-contained computations.
\begin{lemma}
If all basic operations
and basic communications spend constant energy for a fixed space-time
extent (volume), then: 
  \begin{align*}
    E_U(\pi^*) &= O(V_U(\pi^*)) & E_U(\mu) &= O(L^V_U(\mu)) .
  \end{align*}
\end{lemma}
One must spend energy to conserve a memory state, or to perform a
basic operation (in a classical computer). We may assume the constant
complexity of primitive operations, which holds in \prettyref{def:lattice}.  
Let us also assume that the space complexity
of a program is proportional to how much mass is required. Then, the
energy from the resting mass of an optimal computation may be taken
into account, which we call total energy complexity (in metric units):
\begin{lemma}
  \label{lem:toteng}
  \begin{align*}
    E_t(\pi^*) &= d_eV_U(\pi^*) + S_U(\pi^*)d_mc^2 \\
    E_t(\mu) &= d_eL^V_U(\mu) + S_U(\mu)d_mc^2 = O(L^V_U(\mu) + S_U(\mu))
  \end{align*}
\end{lemma}
where $c$ is the speed of light, energy density $d_e = e_u/v_u$, and
mass density $d_m = m_u/s_u$ for the graphical model of complexity.

\begin{lemma}
Conceptual jump total energy (CJTE) of a stochastic source is:
\begin{equation}
\CJTE(\mu) \triangleq E_t(\pi^*).2^{H_U(\mu)} \leq E_t(\mu) \leq 2.CJTE(\mu) .
\end{equation}
\end{lemma}

As a straightforward consequence of the above lemmas, we show a lower
bound on the energy required, that is related to the volume, and
space linearly, and algorithmic complexity of a stochastic source
exponentially, for optimal induction.
\begin{theorem} 
$\CJTE(\mu) = \left( d_eL^V_U(\mu) + S_U(\mu)d_mc^2 \right)
  2^{H_U(\mu)} \leq E_t(\mu) \leq 2.CJTE(\mu)$
\end{theorem}
\begin{proof}
We assume that the energy density is constant; we can 
use $E_t(\cdot)$ for resource bounds in Levin search. 
The inequality is obtained by substituting \prettyref{lem:toteng} into the
definitional inequality.
\end{proof}
The last inequality gives bounds for the total energy cost of
inferring a source $\mu$
in relation to space-time extent (volume of computation), space complexity, and
an exponent of algorithmic complexity of $\mu$. 
This inspires us to define priors using
$\CJV$, $\CJE$, and $\CJTE$ 
which would extend Levin's ideas about resource bounded
Kolmogorov complexity, such as $K_t$ complexity. In the first installation of
ultimate intelligence series, we had introduced complexity measures and priors
based on energy and action. We now define the one that corresponds to CJE
and leave the rest as future work due to lack of space.

\begin{definition}
Energy-bounded algorithmic entropy of a bit string is defined as:
\begin{equation}
H_e(x) \triangleq \min\{|\pi| + \log_2 E_U (\pi) \ | \  U(\pi) = x\}
\end{equation}
\end{definition}

\subsection{Physical limits, incremental learning, and digital physics}

Landauer's limit is a thermodynamic lower bound of $kTln2$ J for 
erasing $1$ bit where $k$ is the Boltzmann constant and $T$ is the temperature
\cite{Landauer:1961}. The total number of bit-wise operations that a
quantum computer can evolve is $2E/h$ operations where $E$ is average
energy, and thus the physical limit to energy efficiency of
computation is about $3.32 \times 10^{33} $ operations/J
\cite{levitin-limit}. Note that the Margolus-Levitin limit may be
considered a quantum analogue of our relation of the volume of
computation with total energy, which is called $E.t$ ``action
volume'' in their paper, as it depends on the quantum of action $h$
which has $E.t$ units. Bremermann discusses the minimum 
energy requirements of computation and communication in 
\cite{Bremermann1982}.
Lloyd \cite{Lloyd:Ultimate} assumes that all
the mass may be converted to energy and calculates the maximum
computation capacity of a 1 kilogram ``black-hole computer'',
performing $10^{51}$ operations over $10^{31}$ bits.  According to an
earlier paper of his, the whole universe may not have performed more
than $10^{120}$ operations over $10^{90}$ bits \cite{Lloyd:Universal}.

\begin{corollary}
  $H(\mu) \leq 397.6 $ for any $\mu$ where the logical volume is $1$.
\end{corollary}
\begin{proof}
   $V(\mu) \leq L^V_U(\mu). 2^{H_U(\mu)+1} \leq 10^{120}$.  Assume
   that $L^V_U(\mu)=1$. 
   \footnote{Although the assumption that it takes only
   1 unit of space-time volume to simulate the minimal program that
   reproduces the pdf $\mu$ is not realistic, we are only considering
   this for the sake of simplicity, 
   and because 1 $m^3$ is close to the volume of a personal computer,
   or a brain. For many pdfs, it could be much larger in practice.}
   $\log_2( 2^{H_U(\mu)+1} )  \leq 3.321 \times 120$.
    $H(\mu)+1  \leq 398.6$
\end{proof}
Therefore, if $\mu$ has a greater algorithmic complexity than about
$400$ bits, it would have been unguaranteed to discover it without any a
priori information.
Digital physics theories suggest that the physical law could be much
simpler than that however, as there are very simple universal
computers in the literature \cite{Miller:2005}, a survey of which
may be found in \cite{neary-smalluniversal}, which means interestingly
that the universe may have had enough time to discover its basic law. 

This limit shows the remarkable importance of incremental learning as
both Solomonoff \cite{solomonoff-agi10} and Schmidhuber \cite{oops}
have emphasized, which is part of ongoing research. We proposed
previously that incremental learning is an AI axiom
\cite{ozkural-diverse}. 
Optimizing energy 
efficiency of computation would also be an obviously useful goal for a
self-improving AI. This measure was first formalized by Solomonoff in
\cite{solomonoff-progress}, which he imagined would be optimizing
performance in units of bits/sec.J as applied to inductive inference, 
which we agree with, and will
eventually implement in our Alpha Phase 2 machine; Alpha Phase
1 has already been partially implemented in our parallel incremental
inductive inference system \cite{teramachine-agi11}.

\section*{Acknowledgements}

Thanks to anonymous reviewers whose comments substantially improved
the presentation. 
Thanks to Gregory Chaitin and Juergen Schmidhuber
for inspiring the mathematical philosophy / digital 
physics angle in the paper. 
I am forever indebted for the high-quality
research coming out of IDSIA which revitalized interest in human-level
AI research.

\bibliographystyle{splncs03} \bibliography{agi,physics,complexity}

\end{document}